\documentclass[letterpaper]{article}
\usepackage{uai2020}
\usepackage{times} 
\usepackage[margin=1in]{geometry}
\usepackage{natbib}
\usepackage{graphicx,amsmath,amssymb}
\usepackage{color}
\usepackage[dvipsnames]{xcolor}
\usepackage{siunitx}
\usepackage{hyperref}
\usepackage{booktabs}
\usepackage{amsfonts} 
\usepackage{subcaption}
\usepackage{caption}
\usepackage[linesnumbered,algoruled,boxed,lined,ruled,algo2e]{algorithm2e}
\usepackage{times}
\usepackage{tikz-cd}
\usetikzlibrary{shapes,mindmap,shapes.arrows,snakes}
\usetikzlibrary{arrows}
\usetikzlibrary{fit,calc}
\newcommand*{\tikzmk}[1]{\tikz[remember picture,overlay,] \node (#1) {};\ignorespaces}
\newcommand{\boxit}[1]{\tikz[remember picture,overlay]{\node[yshift=3pt,fill=#1,opacity=.25,fit={(A)($(B)+(.975\linewidth,.8\baselineskip)$)}] {};}\ignorespaces}
\usepackage{textcomp}
\newcommand{\BlackBox}{\rule{1.5ex}{1.5ex}}  
\newenvironment{proof}{\par\noindent{\bf Proof\ }}{\hfill\BlackBox\\[2mm]}
\newtheorem{example}{Example}
\newtheorem{theorem}{Theorem}
\newtheorem{lemma}[theorem]{Lemma}

\newtheorem{remark}[theorem]{Remark}

\makeatletter
\newcommand\@erelb@r[1]{%
  \mathrel{\tikz[baseline=-.5ex]\draw[#1] (0,0)--(.5,0);}
}
\newcommand{\erelbar}[1]{\@erelbar#1}
\def\@erelbar#1#2{%
  \ifcase\numexpr#1*4+#2\relax
    \@erelb@r{-}\or     
    \@erelb@r{->}\or    
    \@erelb@r{-|}\or    
    \@erelb@r{-o}\or   
    \@erelb@r{<-}\or    
    \@erelb@r{<->}\or   
    \@erelb@r{<-|}\or   
    \@erelb@r{<-o}\or   
    \@erelb@r{|-}\or    
    \@erelb@r{|->}\or   
    \@erelb@r{|-|}\or   
    \@erelb@r{|-o}\or 
    \@erelb@r{o-}\or   
    \@erelb@r{o->}\or  
    \@erelb@r{o-|}\or  
    \@erelb@r{o-o}    
  \else
    \@wrong
  \fi
}
\title{Learning LWF Chain Graphs: A Markov Blanket Discovery Approach}

\author{} 

\author{ {\bf Mohammad Ali Javidian}\\
\And
{\bf Marco Valtorta} \\
Computer Science \& Engineering Department \\
University of South Carolina\\
Columbia, SC 29201 \\
\And
{\bf Pooyan Jamshidi}  \\
}
\begin{document}

\maketitle

\begin{abstract}
This paper provides a graphical characterization of Markov blankets in chain
graphs (CGs) under the Lauritzen-Wermuth-Frydenberg (LWF) interpretation.  The characterization is different from the well-known one for Bayesian networks and generalizes it.  We provide a novel scalable and sound algorithm
for Markov blanket discovery in LWF CGs and prove that the Grow-Shrink algorithm, the IAMB algorithm, and its variants are still correct for Markov blanket discovery in LWF CGs under the same assumptions as for Bayesian networks. We provide a sound and scalable constraint-based framework for learning the structure of LWF CGs from faithful causally sufficient data and prove its correctness when the Markov blanket discovery algorithms in this paper are used. Our proposed algorithms compare positively/competitively against the state-of-the-art  LCD (\textbf{L}earn \textbf{C}hain graphs via \textbf{D}ecomposition) algorithm, depending on the algorithm that is used for Markov blanket discovery. Our proposed algorithms make a broad range of inference/learning problems computationally
tractable and more reliable because they exploit locality.
\end{abstract}

\section{INTRODUCTION}\label{intro}
Probabilistic graphical models are now widely accepted as a powerful and
mature tool for reasoning and decision making under uncertainty.
A \textit{probabilistic graphical model} (PGM) is a compact representation of a joint probability
distribution, from which we can obtain marginal and conditional probabilities \citep{sucar15}. In fact, any PGM consists of two main components: (1) a graph that defines the structure of that model; and (2) a joint distribution over random variables of the model.
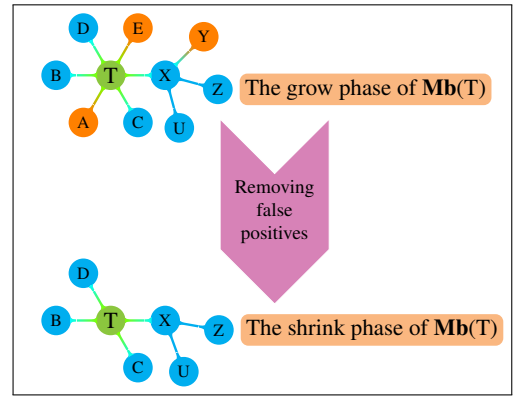
\begin{figure}[!t]
    \centering
    \fbox{
    \tikzset{every concept/.style={minimum size=.5cm, text width=.25cm}, every annotation/.style={concept color=blue,fill=Apricot, text width={}, align=left, font = \large}}
    \tikzset{product size/.style={minimum width=2cm, 
    minimum height=0.5cm,
  },
  product/.style={
    draw,signal, 
    signal to=south, 
    signal from=north,
    product size,
    fill=gray!50!black,
    draw=gray!50!white,
    text=black,
  }, 
}
\scalebox{.725}{%
\begin{tikzpicture}
\path[mindmap,concept color=LimeGreen,text=black]
node[concept] (a) at (0,3) {T}
[clockwise from=0]
child[concept color=cyan,level distance=10mm] {
node[concept] {X}
[clockwise from=45]
child[concept color=orange,level distance=10mm] { node[concept] {Y} }
child[level distance=10mm] { node[concept] (z) {Z} }
child[level distance=10mm] { node[concept] {U} }
}
child[concept color=cyan,level distance=10mm] {node[concept] {C}}
child[concept color=orange,level distance=10mm] { node[concept] {A} }
child[concept color=cyan,level distance=10mm] { node[concept] {B} }
child[concept color=cyan,level distance=10mm] { node[concept] {D} }
child[concept color=orange,level distance=10mm] { node[concept] {E} };
\node [annotation,right] at (z.east)
{The grow phase of \textbf{Mb}(T)};

\node[product, fill=Thistle, draw = white, anchor=west, align=center] (b) at (2,.5) {Removing\\ false\\ positives};

\path[mindmap,concept color=LimeGreen,text=black]
node[concept] at (0,-1.5) {T}
[clockwise from=0]
child[concept color=cyan,level distance=10mm] {
node[concept] {X}
[clockwise from=-10]
child[level distance=10mm] { node[concept] (v) {Z} }
child[level distance=10mm] { node[concept] {U} }
}
child[concept color=cyan,level distance=10mm] {node[concept] {C}}
child[concept color=cyan,level distance=10mm,grow=180] { node[concept] {B} }
child[concept color=cyan,level distance=10mm,grow=120] { node[concept] {D} };
\node [annotation,right] at (v.east)
{The shrink phase of \textbf{Mb}(T)};
\end{tikzpicture}}}
    \caption{The procedure of Markov blanket recovery in the Grow-Shrink based algorithms.}
    \label{fig:GSprocedure}
\end{figure}
Two types of graphical representations of distributions are commonly used, namely, Bayesian networks (BNs) and Markov networks (MNs). Both families encompass the properties of factorization and independencies, but they differ in the set of independencies they can encode and the factorization of the distribution that they induce.

Currently systems containing both causal and non-causal relationships are mostly modeled with \textit{directed acyclic graphs} (DAGs). Chain graphs (CGs) are a type of mixed graphs, admitting both directed and undirected edges, which contain no partially directed cycles. 
So, CGs may contain two types of edges,
the directed type that corresponds to the causal relationship in DAGs and a
second type of edge representing a symmetric relationship \citep{sonntagpena15}. \textit{LWF Chain graphs} were introduced by \textbf{L}auritzen, \textbf{W}ermuth and \textbf{F}rydenberg \citep{f, lw} as a generalization of graphical models based on undirected graphs and DAGs and widely studied \citep{l,lr, d,studeny09,sonntagpena15,roverato05,roverato06}.
From the \textit{causality} point of view, in an LWF CG: Directed edges represent \textit{direct causal effect}s. Undirected edges represents causal effects due to \textit{interference} \citep{Shpitser17,Ogburn18,Bhattacharya19}.

One important and challenging aspect of PGMs is the possibility of learning the structure of models directly
from sampled data. Five \textit{constraint-based} learning algorithms, which use a statistical analysis to test the presence of a
conditional independency, exist for learning LWF CGs: (1) the inductive causation like (IC-like) algorithm \citep{studeny97}, (2) the decomposition-based algorithm called LCD  \citep{mxg}, (3) the answer set programming (ASP) algorithm \citep{sjph}, (4) the inclusion optimal (CKES) algorithm \citep{psn}, and (5) the local structure learning of chain graphs algorithm with false discovery rate control \citep{Wang2019}.

In a DAG $G$ with node set $V$, each local distribution depends only on a single node $v\in V$ and on its parents (i.e., the nodes $u\ne v$ such that $u\erelbar{01}v$, here denoted $pa(v)$). Then the
overall joint density is simply $p(x)=\Pi_{v\in V}^np(x_v|x_{pa(v)})$. The key advantage of the decomposition in this equation is to make \textit{local computations} possible
for most tasks, using just a few variables at a time regardless of the magnitude of $|V|=n$. In Bayesian networks, the concept that enables us to take advantage of local computation is \textit{Markov blanket}. The Markov blanket (\textit{Markov boundary} in Pearl's terminology) of each node $v$, defined as the smallest
set $\textbf{Mb}(v)$ of nodes that separates $v$ from all other nodes $V\setminus \{v,\textbf{Mb}(v)\}$. Markov blankets can be used for variable selection for classification, for causal discovery, and for
Bayesian network learning \citep{Tsamardinos0Mb}. 

Markov blanket discovery has attracted a lot of attention in the context of Bayesian network structure learning (see section \ref{sec:relatedwork}). It is surprising, however,
how little attention (if any) it has attracted in the context of learning LWF chain graphs. 
In this paper, we focus on addressing the problem of Markov blanket discovery for structure learning of LWF chain graphs. For this purpose, we extend the concept of Markov blankets to LWF CGs. We prove that Grow-Shrink Markov Blanket (GSMB) \citep{Margaritis99}, IAMB, and its variants \citep{Tsamardinos0Mb, Yaramakala05} (that are mainly designed for Markov blanket recovery in BNs) are still correct for Markov blanket discovery in LWF CGs under the faithfulness and causal sufficiency assumptions. We propose a new constraint-based Markov blanket recovery algorithm, called MBC-CSP, that is specifically designed for Markov blanket discovery in LWF CGs. 

Since constraint-based learning algorithms are sensitive to \textit{error propagation} \citep{Triantafillou14}, and an erroneous identification of an edge can
propagate through the network and lead to erroneous edge identifications or conflicting orientations \emph{even in seemingly unrelated parts of the network},  the learned chain graph model will be unreliable. In order to address the problem of reliable structure learning, we present a generic approach (i.e., the algorithm is independent of any particular search strategy for Markov blanket discovery) based on Markov blanket recovery to learn the structure of LWF CGs from a faithful data. This algorithm first learns the Markov blanket of each node. This preliminary step greatly simplifies the identification
of neighbours. This in turn results in a significant reduction in the number
of conditional independence tests, and therefore of the overall computational
complexity of the learning algorithm.  
In order to show the effectiveness of this approach,
the resulting algorithms are contrasted against LCD on simulated data. We report experiments showing that our proposed generic algorithm (via 6 different instantiations) provides competitive/better performance against the LCD algorithm in our Gaussian experimental settings, depending on the approach that is used for Markov blanket discovery. Our proposed approach has an advantage over LCD because local structural learning in the
form of Markov blanket is a theoretically well-motivated and
empirically robust learning framework that can serve as a powerful tool in classification and causal discovery \citep{Aliferis10a}.  We also note that Markov blankets are useful in their own right, for example in sensor validation and fault analysis \citep{Sucar1996}. Code for reproducing our results and its corresponding user manual is available at \url{https://github.com/majavid/MbLWF2020}. Our main theoretical and empirical contributions are as follows:

(1) We extend the concept of Markov blankets to LWF CGs and we prove what variables make up the Markov blanket of a target variable in  an  LWF  CG (Section \ref{sec:mbalgs}). 
\newline
(2) We theoretically prove that the Grow-Shrink, IAMB algorithm and its variants are still sound for Markov blanket discovery in LWF chain graphs under the faithfulness and causal sufficiency  assumptions (Section \ref{sec:mbalgs}).
\newline
(3) We present a new algorithm, called MBC-CSP, for learning Markov blankets in LWF chain graphs, and we prove its correctness theoretically (Section \ref{sec:mbalgs}).
\newline
(4) We propose a generic algorithm for structure learning of LWF chain graphs based on the proposed Markov blanket recovery algorithms in Section \ref{sec:mbalgs}, and we prove its correctness theoretically (Section \ref{sec:cglearn}).
\newline
(5) We evaluate the performance of 6 instantiations of the proposed generic algorithm with 6 different Markov blanket recovery algorithms on synthetic Gaussian data, and we show  the competitive performance of our method against the LCD algorithm (Section \ref{sec:eval}).

\section{RELATED WORK}\label{sec:relatedwork}
\textbf{Markov Blanket Recovery for Bayesian Networks with Causal Sufficiency Assumption.}
\cite{Margaritis99} presented the first provably correct algorithm, called Grow-Shrink Markov Blanket (GSMB), that discovers the Markov blanket
of a variable from a faithful data under the causal sufficiency assumption.
Variants of GSMB were proposed to improve speed and reliability such as the Incremental Association Markov Blanket (IAMB) and its variants \citep{Tsamardinos0Mb}, Fast-IAMB \citep{Yaramakala05}, and IAMB with false discovery rate control (IAMB-FDR) \citep{Pena08Mb}. Since in discrete data the sample size required for high-confidence statistical tests of conditional
independence in GSMB and IAMB algorithms grows exponentially in the size of the Markov blanket, several sample-efficient algorithms e.g., HITON-MB \citep{Aliferis10a} and Max–Min Markov Blanket (MMMB) \citep{Tsamardinos2006} were proposed to overcome the data inefficiency of GSMB and IAMB algorithms. One can find alternative computational methods for Markov blanket discovery that were developed in the past two decades in \citep{PENA2007mb,Liu2016,Ling19}, among others.

\noindent\textbf{Markov Blanket Recovery without Causal Sufficiency Assumption.}
Gao and Ji \citep{GaoJi16} proposed the latent Markov blanket learning with constrained structure EM algorithm (LMB-CSEM) to discover the Markov blankets in BNs in the presence of unmeasured confounders. However, LMB-CSEM
was proposed to find the Markov blankets in a DAG and provides no theoretical guarantees for finding all possible unmeasured confounders in the Markov blanket of the target variable. Recently, Yu et. al. \citep{Yu18} proposed a new algorithm, called M3B, to mine Markov blankets in BNs in the presence of unmeasured confounders.

In this paper, we extend the concept of Markov blankets to LWF CGs, which is different from Markov blankets defined in DAGs under the causal sufficiency assumption and also  is different from Markov blankets defined in maximal ancestral graphs without assuming causal sufficiency. So, we need
new algorithms that are specifically designed for Markov blanket discovery
in LWF CGs.

\section{DEFINITIONS AND CONCEPTS}

Below, we briefly list some of the central concepts used in this paper.

A \textit{route} $\omega$ in $G$ is a sequence of nodes (vertices) $v_1,v_2,\dots,v_n, n\ge 1,$ such that $\{v_i,v_{i+1}\}$ is an edge in $G$ for every $1\le i < n$. A \textit{section} of a route is a maximal (w.r.t. set inclusion) non-empty set of nodes $v_i−\cdots−v_j$ s.t. the route $\omega$ contains the subroute $v_i-\cdots -v_j$. It is called a \textit{collider section} if $v_{i-1}\to v_i-\cdots-v_j\gets v_{j+1}$ is a subroute in $\omega$. For any other configuration the section is a non-collider section. 
A \textit{path} is a route containing only distinct nodes. 
A \textit{partially directed path} from $v_1$ to $v_n$ in a graph $G$ is a sequence of $n$ distinct vertices $v_1,v_2,\dots,v_n (n\ge 2)$, such that
\newline
(a) $\forall i (1\le i\le n)$ either $v_i-v_{i+1}$ or $v_i\to v_{i+1}$, and
\newline
(b) $\exists j (1\le j\le n)$ such that $v_j\to v_{j+1}$.
\newline
A partially directed path with $n\ge 3$ and $v_{n}\equiv v_1$ is called a \textit{partially directed cycle}.
If there is a partially directed path from  $a$ to $b$ but not $b$ to $a$, we say that $a$ is an \textit{ancestor} of $b$. The set of ancestors of $b$ is denoted by $an(b)$, and we generalize the definition to a set of nodes in the obvious way.

Formally, we define the set of parents, children, neighbors, and spouses of a variable (node) in an LWF CG $G=(V,E)$ as follows, respectively: $pa(v)=\{u\in V|u\to v \in E\}$, $ch(v)=\{u\in V|v\to u \in E\}$, $ne(v)=\{u\in V|v- u \in E\}$, $sp(v)=\{u\in V|\exists w\in V \textrm{ s.t. }u\to w\gets v \textrm{ in }G\}$. The boundary $bd(A)$ of a subset $A$ of vertices is the set of vertices in $V\setminus A$ that are parents or neighbors to vertices in $A$. The closure of $A$ is $cl(A)=bd(A)\cup A$. If $bd(a)\subseteq A$, for all $a\in A$ we say that $A$ is an \textit{ancestral set}. The smallest ancestral set containing $A$ is denoted by $An(A)$.

An \textit{LWF CG} is a graph in which there are no partially directed cycles. The chain components $\mathcal{T}$ of a CG are the connected components of the undirected
graph obtained by removing all directed edges from the CG. A \textit{minimal complex} (or simply a complex) in a CG is an induced subgraph of the form $a\to v_1-\cdots \cdots-v_r\gets b$. We say that $a$ is a \textit{complex-spouse} of $b$ and vice versa, and that $csp(a)=\{b\in V|\exists\textrm{ a minimal complex of form }a\to x-\cdots- y\gets  b\}$. The \textit{skeleton} of an LWF CG $G$ is obtained from $G$ by changing all directed edges of $G$ into undirected edges. 
For a CG $G$ we define its \textit{moral graph} $G^m$ as the undirected  graph with the same vertex set but with $\alpha$ and $\beta$ adjacent in $G^m$ if and 
only if either $\alpha \to \beta$ , $\beta\to \alpha$, $\alpha - \beta$, or if $\alpha\in csp(\beta)$. 

\emph{Global Markov property for LWF CGs:} 
	For any triple
	$(A, B,S)$ of disjoint subsets of $V$ such that $S$ separates $A$ from $B$
	in $(G_{An(A\cup B\cup S)})^m$, the moral graph of the smallest ancestral set containing $A\cup B\cup S$, indicated as $A {\!\perp\!\!\!\perp}_c B | S$ (read: $S$ \textit{$c$-separates} $A$ from $B$ in the CG $G$), we have $A {\!\perp\!\!\!\perp}_p B | S$, i.e., $A$ is independent of $B$ given $S$. In words, if $S$ $c$-separates $A$ from $B$ in the CG $G$, then  $A$ and $B$ are independent given $S$.
An equivalent path-wise $c$-separation criterion, which generalizes the $d$-separation criterion for DAGs, was introduced in \citep{studeny98}. A route $\omega$ is \textit{active} with respect to a
set $S\subseteq V$ if (i) every collider section of $\omega$ contains a node of $S$ or $an(S)$, and (ii) every node in a non-collider section on the route is not in $S$. A route which is not active with respect to $S$ is \textit{intercepted} (blocked) by $S$. If $G$ is an LWF CG then $X$ and $Y$ are $c$-separated given $S$ iff there exists no active route between $X$ and $Y$.

We say that two LWF chain graphs are \textit{Markov equivalent} if they induce the same conditional independence restrictions.  Two chain graphs are Markov equivalent if and only if they have the same skeletons and the same minimal complexes \citep{f}. Every class of Markov equivalent CGs has a unique CG, called the \textit{largest CG}, with the greatest number of undirected edges \citep{f}. 

The \textit{Markov condition} is said to hold for a DAG $G = (V,E)$ and a probability distribution $P(V)$ if every variable $T$ is statistically independent of its graphical non-descendants (the set of vertices for which there is no directed path from $T$) conditional on its graphical parents
in $P$. Pairs $\langle G,P\rangle$ that satisfy the Markov condition satisfy the following implication:
$\forall X,Y\in V, \forall Z\subseteq V\setminus\{X,Y\}:(X{\!\perp\!\!\!\perp}_{d} Y |Z \Longrightarrow X{\!\perp\!\!\!\perp}_{p} Y |Z)$.
The \textit{faithfulness condition} states that the only conditional independencies to hold are those specified by the Markov condition, formally:
$\forall X,Y\in V, \forall Z\subseteq V\setminus\{X,Y\}: (X{\not\!\perp\!\!\!\perp}_{d} Y |Z \Longrightarrow X{\not\!\perp\!\!\!\perp}_{p} Y |Z)$.

Let a Bayesian network $G = (V,E,P)$ be given. Then,  $V$ is a set of random variables, $(V,E)$ is a DAG, and $P$ is a joint probability
distribution over $V$.  Let $T\in V$. Then the \textit{Markov blanket} $\textbf{Mb}(T)$ is the set of all parents of $T$, children of $T$, and spouses of $T$.  Formally, $\textbf{Mb}(T)=pa(T)\cup ch(T) \cup sp(T)$.

\section{MARKOV BLANKET DISCOVERY IN LWF CHAIN GRAPHS }\label{sec:mbalgs}
Let $G = (V,E,P)$ be an LWF chain graph model. Then, $V$ is a set of random variables, $(V,E)$ is an LWF chain graph, and $P$ is a joint probability
distribution over $V$.  Let $T\in V$. Then the \textit{Markov blanket} $\textbf{Mb}(T)$ is the set of all parents of $T$, children of $T$, neighbors of $T$, and complex-spouses of $T$. Formally, $\textbf{Mb}(T)=bd(T)\cup ch(T)\cup csp(T)$.
We first show that the Markov blanket of the target variable $T$ in an LWF CG probabilistically shields $T$ from the rest of the variables. Under the faithfulness assumption, the Markov blanket is the smallest set with this property. 
Then, we propose a novel algorithm, called MBC-CSP, that is specifically designed for Markov blanket discovery in LWF CGs. In addition, we prove that GSMB, IAMB and its variants, and MBC-CSP are sound for Markov blanket discovery in LWF CGs under the faithfulness and causal sufficiency assumptions.

\begin{theorem}\label{thm:mbcg}
Let $G=(V,E,P)$ be an LWF chain graph model.
Then, $T{\!\perp\!\!\!\perp}_p V\setminus\{T,\textbf{Mb}(T)\}|\textbf{Mb}(T)$.
\end{theorem}
\begin{proof}
It is enough to show that for any $A\in V\setminus\{T,\textbf{Mb}(T)\}$, $T{\!\perp\!\!\!\perp}_cA|\textbf{Mb}(T)$. For this purpose, we prove that any route between $A$ and $T$ in $G$ is blocked by $\textbf{Mb}(T)$. In the following cases ($A-{\!\!\!*}B$, where means $A-B$ or $A\to B$ and $A{{*\!\!\!}-{\!\!\!*}}B$ means $A-B$, $A\to B$, or $A\gets B$),
we assume without loss of generality that $T$ cannot appear between $A$ and $B$. (If $T$ appears between $A$ and $B$, the argument for the appropriate case can be applied inductively.)
\newline
\textbf{(1)} The route $\omega$ between $A$ and $T$ is of the form $A{{*\!\!\!}-{\!\!\!*}}\cdots {{*\!\!\!}-{\!\!\!*}}B\to T$. Clearly, $B$ blocks the route $\omega$.
\newline
\textbf{(2)} The route $\omega$ between $A$ and $T$ is of the form $A{{*\!\!\!}-{\!\!\!*}}\cdots {{*\!\!\!}-{\!\!\!*}}B-T$. Clearly, $B$ blocks the route $\omega$.
\newline
\textbf{(3)} The route $\omega$ between $A$ and $T$ is of the form $A{{*\!\!\!}-{\!\!\!*}}\cdots{{*\!\!\!}-{\!\!\!*}} C{{*\!\!\!}-{\!\!\!*}}B\gets T$. We have the following subcases:
\newline
\textit{(3i)} The route $\omega$ between $A$ and $T$ is of the form $A{{*\!\!\!}-{\!\!\!*}}\cdots{{*\!\!\!}-{\!\!\!*}} C\gets B\gets T$. Clearly, $B$ blocks the route $\omega$.
\newline
\textit{(3ii)} The route $\omega$ between $A$ and $T$ is of the form $A{{*\!\!\!}-{\!\!\!*}}\cdots{{*\!\!\!}-{\!\!\!*}} C-B\gets T$. If $B$ is \textit{not} a node on a collider section of $\omega$, $B$ blocks the route $\omega$. However, If $B$ is a node on a collider section of $\omega$, there are nodes $D$ and $E$ ($\ne A,T$) s.t. the route $\omega$ has the form of $A{{*\!\!\!}-{\!\!\!*}}\cdots{{*\!\!\!}-{\!\!\!*}}E\to D-\cdots- C-B\gets T$. $E\in csp(T)$ blocks the route $\omega$.
\newline
\textit{(3iii)} The route $\omega$ between $A$ and $T$ is of the form $A{{*\!\!\!}-{\!\!\!*}}\cdots{{*\!\!\!}-{\!\!\!*}} C\to B\gets T$. $C\in sp(T)$ blocks the route $\omega$.
\newline 
From the global Markov property it follows that every $c$-separation relation in
$G$ implies conditional independence in every joint probability distribution $P$ that satisfies the
global Markov property for $G$. Thus, we have $T{\!\perp\!\!\!\perp}_p V\setminus\{T,\textbf{Mb}(T)\}|\textbf{Mb}(T)$.
\end{proof}
\begin{example}
Suppose $G$ is the LWF CG in Figure \ref{fig:mbex}). $\textbf{Mb}(T)=\{C,F,G,H,K,L\}$, because $pa(T)=\{C,G\}, ch(T)=\{K\}, ne(T)=\{F\}, csp(T)=\{L,H\}$. Note that if only $T$'s adjacents are instantiated, then $T$ is not $c$-separated from $L$ and $H$ in $G$.  
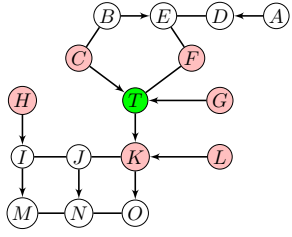
\begin{figure}
    \centering
    \scalebox{.75}{
    \begin{tikzpicture}
\tikzset{vertex/.style = {shape=circle,draw,minimum size=1.5pt,inner sep=1pt}}
\tikzset{edge/.style = {->,> = latex',thick}}


\node[vertex] (a) at (4.5,3.5)  {$A$};
\node[vertex] (b) at (1.5,3.5)  {$B$}; 
\node[vertex,fill=pink] (c) at (1,2.75)  {$C$};
\node[vertex] (d) at (3.5,3.5){$D$};
\node[vertex] (e) at (2.5,3.5){$E$};
\node[vertex,fill=pink] (f) at (3,2.75){$F$}; 
\node[vertex,fill=pink] (g) at (3.5,2){$G$};
\node[vertex,fill=pink] (h) at (0,2){$H$};
\node[vertex] (i) at (0,1){$I$};
\node[vertex] (j) at (1,1){$J$};
\node[vertex,fill=pink] (k) at (2,1){$K$};
\node[vertex,fill=pink] (l) at (3.5,1){$L$};
\node[vertex] (m) at (0,0){$M$};
\node[vertex] (n) at (1,0){$N$};
\node[vertex] (o) at (2,0){$O$};
\node[vertex,fill=green] (t) at (2,2){$T$};


\draw[edge] (a) to (d);
\draw[edge] (b) to (e);
\draw[edge] (c) to (t);
\draw[edge] (g) to (t);
\draw[edge] (t) to (k);
\draw[edge] (k) to (o);
\draw[edge] (l) to (k);
\draw[edge] (j) to (n);
\draw[edge] (i) to (m);
\draw[edge] (h) to (i);

\draw[thick] (d) to (e);
\draw[thick] (e) to (f);
\draw[thick] (b) to (c);
\draw[thick] (f) to (t);
\draw[thick] (i) to (j);
\draw[thick] (j) to (k);
\draw[thick] (m) to (n);
\draw[thick] (n) to (o);
\end{tikzpicture}}
    \caption{The LWF CG $G$. The Markov blanket of the target node $T$ is $\textbf{Mb}(T)=\{C,F,G,H,K,L\}$.}
    \label{fig:mbex}
\end{figure}
\end{example}
\subsection{The MBC-CSP Algorithm for Markov Blanket Discovery in LWF CGs}
The MBC-CSP algorithm is structurally similar to the standard Markov blanket discovery algorithms and follows the same two-phase \textit{grow-shrink} structure as shown in the Figure \ref{fig:GSprocedure}. An estimate of the $\textbf{Mb}(T)$ is kept in the set $\text{\it CMB}$. In the grow phase all variables that
belong in $\textbf{Mb}(T)$ and possibly more (\textit{false positives}) enter
$\text{\it CMB}$ while in the shrink phase the false positives are identified and removed so that $\text{\it CMB} = \textbf{Mb}(T)$ in the end.   

In the grow phase, MBC-CSP first recovers $adj(T):=pa(T)\cup ch(T)\cup ne(T)$, i.e., the variables adjacent to $T$. This step is similar to AdjV algorithm in \citep{Yu18}. Then it discovers complex-spouses of $T$ denoting by $csp(T)$. In the shrink phase, MBC-CSP removes one-by-one the elements of $\text{\it CMB}$ that do not belong to the $\textbf{Mb}(T)$ by
testing whether a feature $X$ from $\text{\it CMB}$ is independent of $T$ given the remaining $\text{\it CMB}$.
\begin{algorithm2e}
    \caption{MBC-CSP: An algorithm for Markov blanket discovery in LWF CGs}\label{alg:mmbcsp}
    \SetAlgoLined
    \SetNoFillComment
    \footnotesize\KwIn{a data set with variable set $V$, target variable $T$, and significance level $\alpha$.}
	\KwOut{$\textrm{\textbf{Mb}}(T)$.}
    \tcc{\textcolor{blue}{\textbf{Phase 1: Grow (Forward)}}}
        
    \tcc{\textcolor{red}{\textbf{step 1:} $adj(T):= pa(T)\cup ch(T)\cup ne(T)$, the set of variables adjacent to $T$.}}
    \tikzmk{A}
    \For{($V_i \in V\setminus\{T\}$)}{
        $p_{V_i}= pvalue(T\perp\!\!\!\perp V_i|\varnothing)$;
        
        \eIf{($p_{V_i} > \alpha$)}{
           $\textbf{Sepset}(T,V_i)= \varnothing$;
           \tcc{$T$ is marginally independent of $V_i$.}
        }{
            Add $V_i$ to $adj(T)$;
        }
    }
    Sort $adj(T)$ in increasing value of $cor(V_i, T )$;
    
    Set $k = 1, \#adj = |adj(T)|$;
    
    \While{($k\le \#adj$)}{
        \For{($V_j\in adj(T)$)}{
            \If{($\exists S\subseteq adj(T)\setminus\{V_j\} s.t. T\perp\!\!\!\perp V_j|S \textrm{ and } |S|=k$)}{
            $adj(T) = adj(T) \setminus V_j$;
            
            $\textbf{Sepset}(T,V_j)= S$;
        }
        }
        $k = k+1,\#adj = |adj(T)|$;
    }\tikzmk{B}\boxit{green!35}
    \tcc{\textcolor{red}{\textbf{step 2:} $csp(T)$, complex-spouses of $T$.}}
    \tikzmk{A}
    \For{$V_i\in adj(T)$}{
        \For{$V_j\in V\setminus\{adj(T),T\}$}{
            $pval_1=pvalue(T\perp\!\!\!\perp V_j|\textbf{Sepset}(T,V_j))$;
            
            $pval_2=pvalue(T\perp\!\!\!\perp V_j|(\textbf{Sepset}(T,V_j))\cup \{V_i\})$;
            
            \If{$pval_1 > \alpha \textrm{ and } pval_2 < \alpha$}{
                Add $V_j$ to $csp(T)$;
            }
        }
    }\tikzmk{B}
    \boxit{pink}
    $\text{\it CMB} =adj(T)\cup csp(T)$;
    
    \tcc{\textcolor{blue}{\textbf{Phase 2: Shrink (Backward)}}}
    \textit{continue = TRUE};
    
    \If{($|\text{\it CMB}| = 0$)}{
        \textit{continue = FALSE};
    }\tikzmk{A}
    \While{(continue)}{
    
        $P_Y = \underset{Y \in \text{\it CMB}}{pvalue}(T\perp\!\!\!\perp Y|\text{\it CMB}\setminus\{Y\})$;
    
    $p.val.max = \max_{Y\in \text{\it CMB}}P_Y$;
    
    $Candidas = \{Y\in \text{\it CMB}|P_Y = p.val.max\}$;
    
        \eIf{($p.val.max > \alpha$)}{
        \tcc{i.e., $T\perp\!\!\!\perp Y|\text{\it CMB}\setminus\{Y\}$}
        
        $\text{\it CMB} = \text{\it CMB}\setminus Candidas[1]$;
        
        \tcc{$Candidas[1]$ means the first element of \textit{Candidas}.}
    }{
        \textit{continue = FALSE};
    }
    }\tikzmk{B}
 \boxit{cyan!60}
    return \textit{\text{\it CMB}};
  \end{algorithm2e}

\textbf{MBC-CSP Description:} In Algorithm \ref{alg:mmbcsp}, $adj(T)$ stores the variables adjacent to $T$, $S$ is the conditioning set, $cor(V_i, T )$ denotes the value of the correlation between $V_i$ and $T$, $\#adj(T)$
is the number of variables in $adj(T)$, and $\textbf{Sepset}(T,V_i)$ means
the separation set for $V_i$ with respect to $T$, i.e., the conditioning set that makes $T$ and $V_i$ conditionally independent.
From line 1 to 8 of Algorithm \ref{alg:mmbcsp}, MBC-CSP removes the
variables that are marginally independent of $T$ and then sorts the remaining
variables in an ascending order of their correlations with $T$. The obtained $adj(T)$ at the end of line 8 may include some false positives. In order to remove false positives from $adj(T)$, we select the variable with the smallest correlation
with $T$, because a variable
with a weak correlation with $T$ may have a higher probability to be removed from $adj(T)$ as a false positive than a variable with a strong correlation with $T$. In this way we speed up the procedure of false positives removal. This procedure begins with a conditioning set of size 1 and then increases the size of the conditioning set one-by-one iteratively until its size is bigger than the size of the current set $adj(T)$.
At each iteration, if a variable is found to be independent of $T$, the variable is removed from the current $adj(T)$ (line 9 to 19 of Algorithm \ref{alg:mmbcsp}). Now, we need to add the complex-spouses of $T$ to the obtained set at the end of line 19. For this purpose, lines 21-29 find the set of $csp(T)$ by checking the following conditions for each $V_j\in V\setminus\{adj(T),T\}$: $T\perp\!\!\!\perp V_j|\textbf{Sepset}(T,V_j)$ and $T\not\perp\!\!\!\perp V_j|(\textbf{Sepset}(T,V_j))\cup \{V_i\})$. According to the global Markov property for LWF CGs, these two conditions together guarantee that $V_j\in csp(T)$. At the end of line 30, we obtain a candidate set for the Markov blanket of $T$ that may contain some false positives. The phase 2 of Algorithm \ref{alg:mmbcsp} i.e., lines 35-44, uses the same idea of the shrinking phase of Markov blanket discovery algorithm IAMB \citep{Tsamardinos0Mb} for the output of phase 1 to reduce the number of false positives in the output of the algorithm. For this purpose, we remove one-by-one the variables that do not belong to the $\textbf{Mb}(T)$ by
testing whether a variable $Y$ from $\text{\it CMB}$ is independent of $T$ given the remaining variables in $\text{\it CMB}$.

\begin{remark}
For the adjacency recovery phase of Algorithm \ref{alg:mmbcsp} (line 1-19), one can use the HITON-PC or MMPC \citep{Aliferis10a} algorithms, especially in cases where a sample-efficient  algorithm is needed.
\end{remark}

\begin{theorem}\label{thm:Mbalgs}
Given the Markov assumption, the faithfulness assumption, a graphical model represented by an LWF CG, and i.i.d. sampling, in the large sample limit, the Markov blanket recovery algorithms GS \citep{Margaritis99}, IAMB \citep{Tsamardinos0Mb}, MMBC-CSP (Algorithm \ref{alg:mmbcsp}), fastIAMB \citep{Yaramakala05}, Interleaved Incremental Association (interIAMB) \citep{Tsamardinos0Mb}, and fdrIAMB \citep{Pena08Mb} correctly identify all Markov blankets for each variable.  (Note that \emph{Causal Sufficiency} is assumed i.e., all causes of more than one variable are observed.)
\end{theorem}
\begin{proof}[Sketch of proof]
If a variable belongs to $\textbf{Mb}(T)$, then it will be admitted in the first step (Grow phase) at some point, since it will be dependent on $T$ given the candidate set of  $\textbf{Mb}(T)$.  This holds because of the faithfulness and because the set $\textbf{Mb}(T)$ is the minimal set with that property. If $X\not\in \textbf{Mb}(T)$, then conditioned on $\textbf{Mb}(T)\setminus\{X\}$, it will be independent of $T$ and thus will
be removed from the candidate set of $\textbf{Mb}(T)$ in the second phase (Shrink phase) because the Markov condition entails that independencies in the distribution are represented in the graph. Since the faithfulness condition entails dependencies in the distribution from the graph, we never remove any variable $X$ from the candidate set of $\textbf{Mb}(T)$ if $X\in\textbf{Mb}(T)$.  Using this
argument inductively we will end up with the $\textbf{Mb}(T)$.
\end{proof}

\section{LEARNING LWF CGs VIA MARKOV BLANKETS}\label{sec:cglearn}
Any sound algorithm for learning Markov blankets of LWF CGs can be employed and extended to a full  LWF CG learning algorithm, as originally suggested in \citep{Margaritis99} for Grow-Shrink Markov blanket algorithm (for Bayesian networks). Thanks to the proposed Markov blanket discovery algorithms listed in Theorem \ref{thm:Mbalgs}, we can now present a generic algorithm for learning LWF CGs. Algorithm \ref{alg:LWFMb} lists pseudocode for the three main phases of this approach.
\begin{algorithm2e}[t]
	\caption{MbLWF: An algorithm for learning LWF CGs via Markov blanket discovery}\label{alg:LWFMb}
	\SetAlgoLined
	\SetNoFillComment
	\small\KwIn{a set $V$ of nodes and a probability distribution $p$ faithful to an unknown LWF chain graph $G=(V,E)$.}
	\KwOut{The pattern of $G$.}
	\tcc{\textcolor{blue}{\textbf{Phase 1: Learning Markov blankets}}}
	For each variable $X_i\in V$, learn its Markov blanket $\textbf{Mb}(X_i)$\;
	Check whether the Markov blankets are symmetric, e.g., $X_i\in \textbf{Mb}(X_j) \leftrightarrow X_j\in \textbf{Mb}(X_i)$. Assume that nodes for whom symmetry does not hold are false positives and drop them from each other's Markov blankets\;
	Set $\textbf{Sepset}(X_i,X_j)=\textbf{Sepset}(X_j,X_i)$ to the smallest of $\textbf{Mb}(X_i)$ and $\textbf{Mb}(X_j)$ if $X_i\not\in \textbf{Mb}(X_j)$ and $X_j\not\in\textbf{Mb}(X_i)$\;
	\tcc{\textcolor{blue}{\textbf{Phase 2: Skeleton Recovery}}}
	Construct the undirected graph $H=(V,E)$, where $E=\{X_i -X_j| X_j\in\textbf{Mb}(X_i) \textrm{ and }X_i\in\textbf{Mb}(X_j)\}$;
	
	\For{$i\gets 0$ \KwTo $|V_H|-2$}{
		\While{possible}{
			Select any ordered pair of nodes $u$ and $v$ in $H$ such that $u\in ad_H(v)$ and $|ad_H(u)\setminus v|\ge i$\;
			\tcc{$ad_H(x):=\{y\in V| x-y \in E\}$}
			\If{\textrm{there exists $S\subseteq (ad_H(u)\setminus v)$ s.t. $|S|=i$ and $u\perp\!\!\!\perp_p v|S$ (i.e., $u$ is independent of $v$ given $S$ in the probability distribution $p$)}}{
				Set $S_{uv} = S_{vu} = S$\;
				Remove the edge $u - v$ from $H$\;
			}
		}
	}
	\tcc{\textcolor{blue}{\textbf{Phase 3: Complex Recovery \citep{mxg}}}}
	Initialize $H^*=H$\;
	\For{\textrm{each vertex pair $\{u,v\}$ s.t. $u$ and $v$ are not adjacent in $H$}}{
		\For{\textrm{each $u- w$ in $H^*$}}{
			\If{$u\not\perp\!\!\!\perp_p v|(S_{uv}\cup \{w\})$}{
				Orient $u-w$ as $u\erelbar{01} w$ in $H^*$\;
			}
		}
	}
	Take the pattern of $H^*$\;
\end{algorithm2e}

\textbf{Phase 1: Learning Markov blankets:} This phase consists of learning the Markov blanket of each variable with feature selection to reduce the number of candidate structures early on. Any algorithm in Theorem \ref{thm:Mbalgs} can be plugged in Step 1.  Once all Markov blankets have been learned, they are checked
for consistency (Step 2) using their symmetry; by definition $X_i\in \textbf{Mb}(X_j) \leftrightarrow X_j\in \textbf{Mb}(X_i)$.
Asymmetries are corrected by treating them as false positives and removing those
variables from each other’s Markov blankets. At the end of this phase, separator sets of $X$ and $Y$ set to the smallest of $\textbf{Mb}(X)$ and $\textbf{Mb}(Y)$ if $X\not\in \textbf{Mb}(Y)$ and $Y\not\in\textbf{Mb}(X)$.

\textbf{Phase 2: Skeleton Recovery:} First, we construct the moral graph of the LWF CG $G$ that is an undirected graph in which each node of the original $G$ is now connected to its Markov blanket (line 4 of Algorithm \ref{alg:LWFMb}). Lines 5-13 learn the \textit{skeleton} of the LWF CG by removing the spurious edges. In fact, we remove the added undirected edge(s) between each variable $T$ and its complex-spouses due to the fact that $csp(T)\subseteq \textbf{Mb}(T)$. Separation sets are updated correspondingly.

\textbf{Phase 3:} \textbf{Complex Recovery:}
We use an approach similar to the proposed algorithm by \citep{mxg} for complex recovery. To get the pattern of $H^*$ in line 22, at each step, we consider a pair of candidate complex arrows $u_1 \to w_1$ and $u_2\to w_2$ with $u_1 \ne u_2$, then we check whether there is an undirected path from $w_1$ to $w_2$ such that none of its intermediate vertices is adjacent to either $u_1$ or $u_2$. If there exists such a path, then $u_1 \to w_1$ and $u_2\to w_2$ are labeled (as complex arrows). We repeat this procedure until all possible candidate pairs are examined. The pattern is then obtained by removing directions of all unlabeled as complex arrows in $H^*$ \citep{mxg}. Note that one can use three basic
rules, namely the \textit{transitivity rule}, the \textit{necessity rule}, and the \textit{double-cycle rule}, for changing the obtained pattern in the previous phase into the corresponding largest CG (see \cite{studeny97} for details).

\textbf{Computational Complexity Analysis of Algorithm \ref{alg:LWFMb}} Assume that the ``learning Markov blankets" phase uses the grow-shrink (GSMB) approach and $n=|V|$, $m=|E|$, where $G=(V,E)$ is the true LWF CG.  Since the Markov blanket algorithm involves $O(n)$ conditional independence (CI) tests, Phase 1 (learning Markov blankets) involves $O(n^2)$ tests. If $b = max_X |\textrm{\textbf{Mb}}(X)|$, the skeleton recovery (line 5-13) does $O(nb2^b)$ CI tests. In the worst case, i.e. when $b = O(n)$ and $m = O(n^2)$ i.e. the original graph is dense, the total complexity for these 2 phases becomes $O(n^2+nb2^b)$ or $O(n^2 2^n)$. Under the assumption that $b$ is bounded by a constant (the sparseness assumption), the complexity of Phase 1 and 2 together is $O(n^2)$ in the number of CI tests. As claimed in \citep{mxg}, the total complexity of Phase 3 (complex recovery, lines 14-22) is $O(mn)$ in the number of CI tests. The total number of CI tests for the entire algorithm is therefore $O(n^2+nb2^b+mn)$. Under the assumption that $b$ is bounded by a constant, this algorithm is $O(n^2+mn)$ in the number of CI tests. 

\section{Experimental Evaluation}\label{sec:eval}
We performed a large set of experiments on simulated data for contrasting: (1) our proposed Markov blanket discovery algorithm, MBC-CSP, against GS, IAMB, fastIAMB, interIAMB, and fdrIAMB for Markov blanket recovery only, due to their important role in causal discovery and classification; and (2) our proposed structure learning algorithms (GSLWF, IAMBLWF, interIAMBLWF, fastIAMBLWF, fdrIAMBLWF,  and MBCCSPLWF) against the state-of-the-art algorithm LCD for LWF CG recovery. 
We implemented all algorithms in R by extending code from the $\mathsf{bnlearn}$ \citep{Scutarijstatsoft09} and $\mathsf{pcalg}$ \citep{JSSpcalg} packages to LWF CGs.
We run our algorithms and the LCD algorithm on randomly generated LWF CGs and we compare the results and report summary error measures.

\textbf{Experimental Settings:}
Let $N=2$ or 3 denote the average degree of edges (including undirected, pointing out, and pointing in) for each vertex. We generated random LWF CGs with 30, 40, or 50 variables and $N=2$ or 3, as described in \citep{mxg} (see Appendix B for details). Then, we generated Gaussian distributions of size 200 and 2000 on the resulting LWF CGs via the $\mathrm{rnorm.cg}$ function from the \href{http://www2.uaem.mx/r-mirror/web/packages/lcd/lcd.pdf}{LCD} R package, respectively.
For each sample, two different
significance levels $(\alpha = 0.05, 0.005)$ are used to perform the hypothesis tests. The \textit{null hypothesis} $H_0$ is ``two variables $u$ and $v$ are conditionally independent given a set $C$ of variables" and alternative $H_1$ is that $H_0$ may not hold. We then
compare the results to access the influence of the significance testing level on the performance of our algorithms.

\textbf{Metrics for Evaluation:} 
We evaluate the performance of the proposed algorithms in terms of the six measurements that are commonly used~\citep{Colombo2014,Tsamardinos2006} for constraint-based algorithms: (a) the true positive
rate (TPR) (also known as recall), (b) the false positive rate (FPR), (c) the true discovery rate (TDR) (also known as precision), (d) accuracy (ACC) for the skeleton, (e) the Structural Hamming Distance (SHD), and (f) run-time. In principle, large values of TPR, TDR, and ACC, and small values of FPR and SHD indicate good performance.  

\begin{figure}
	\centering
	\includegraphics[scale=.474,page=1]{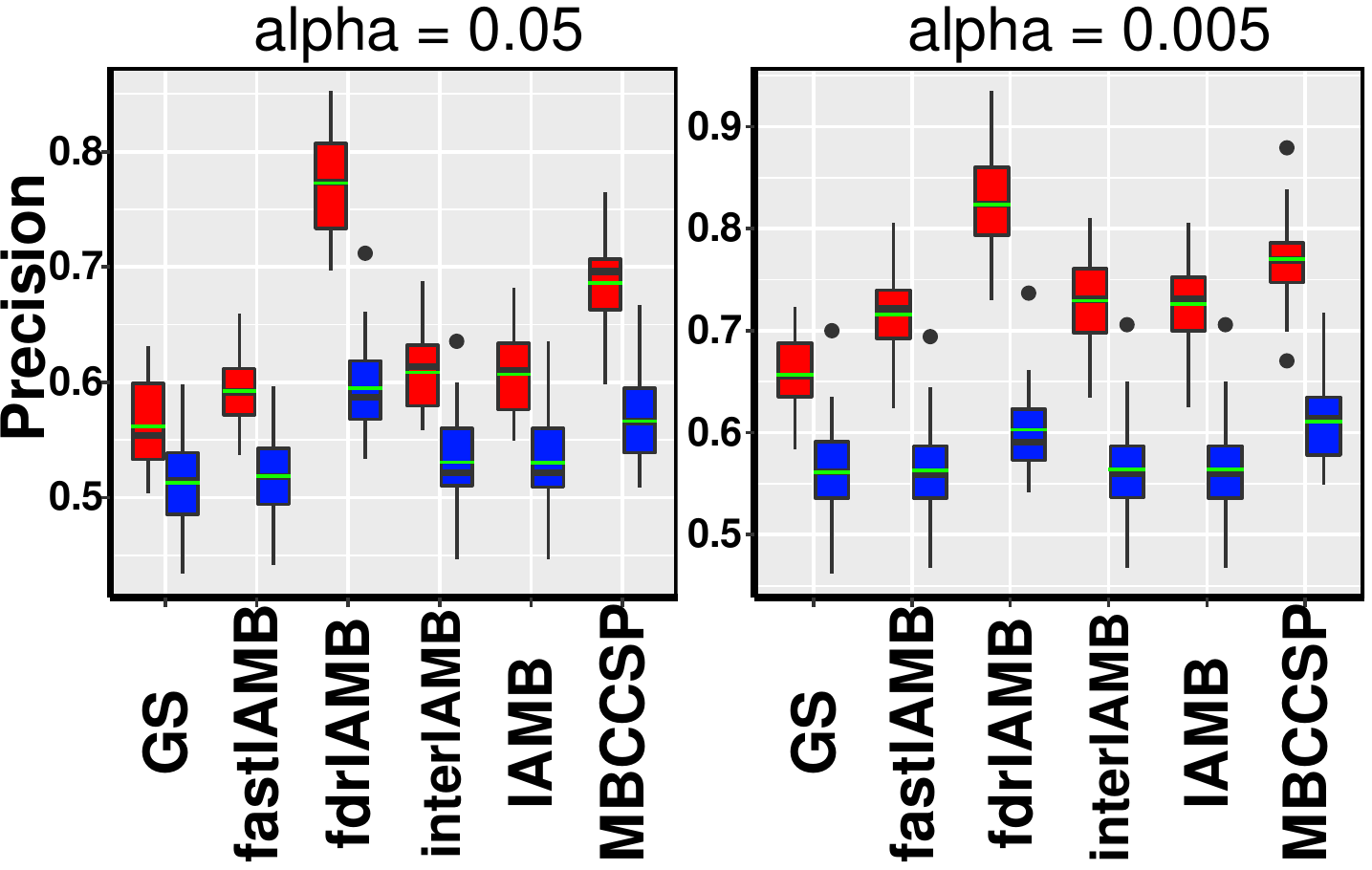}
	\includegraphics[scale=.474,page=2]{images/Mbsboxplots503.pdf}
	\includegraphics[scale=.474,page=3]{images/Mbsboxplots503.pdf}
	\includegraphics[scale=.474,page=4]{images/Mbsboxplots503.pdf}
	\includegraphics[scale=.474,page=5]{images/Mbsboxplots503.pdf}
	\caption{Performance of Markov blanket recovery algorithms for randomly generated Gaussian chain graph models:
		 over 30 repetitions with 50 variables  correspond to N = 3. The green line in a box indicates the mean of that group.}
	\label{fig:503Mbs}
\end{figure}

\begin{figure}
	\centering
	\includegraphics[scale=.474,page=1]{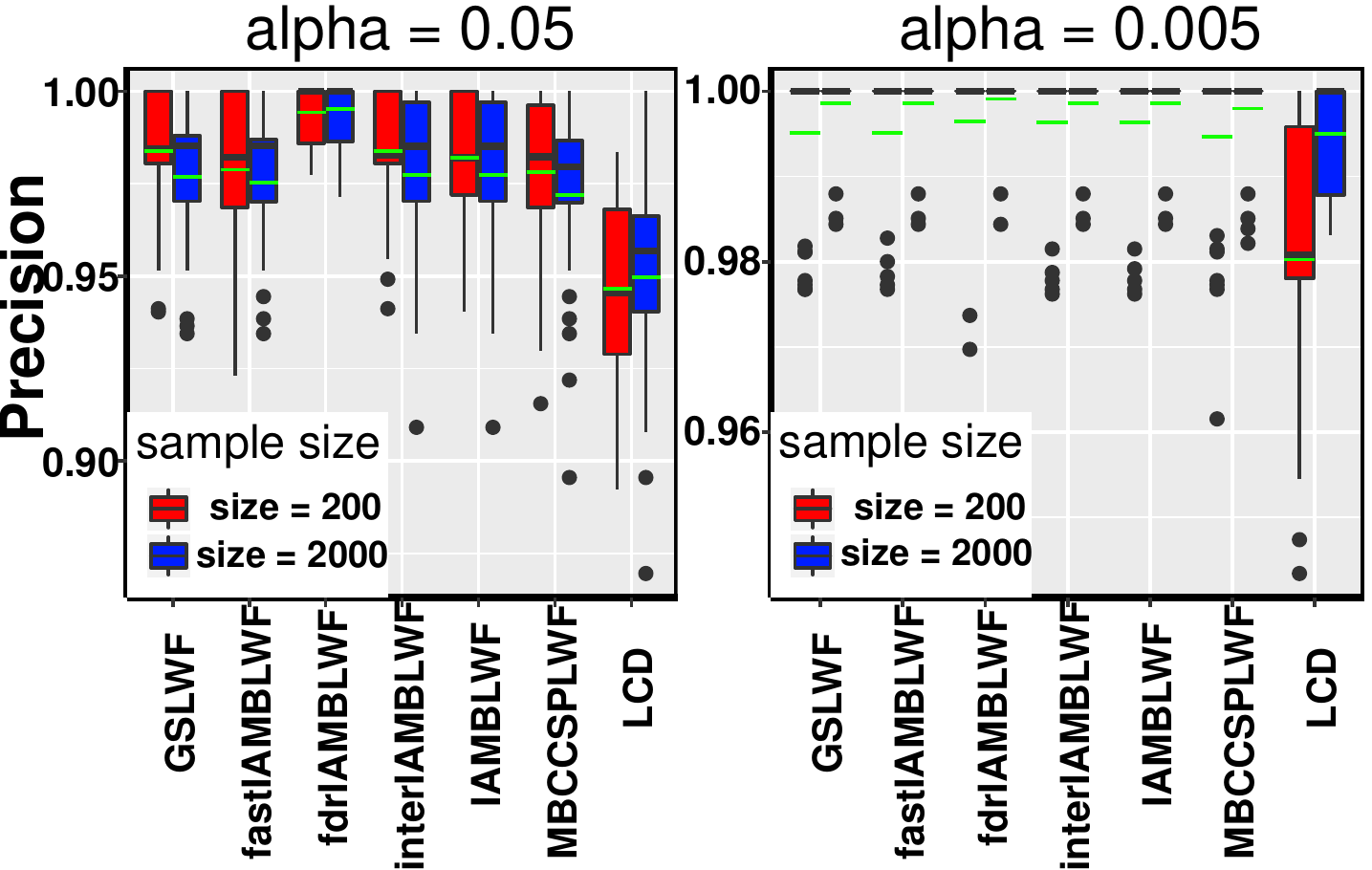}
	\includegraphics[scale=.474,page=2]{images/MBLWFCG503.pdf}
	\includegraphics[scale=.474,page=3]{images/MBLWFCG503.pdf}
	\includegraphics[scale=.474,page=4]{images/MBLWFCG503.pdf}
	\includegraphics[scale=.474,page=5]{images/MBLWFCG503.pdf}
	\caption{Performance of LCD and MbLWF algorithms for randomly generated Gaussian chain graph models:
		 over 30 repetitions with 50 variables  correspond to N = 3. The green line in a box indicates the mean of that group.}
	\label{fig:503cglearn}
\end{figure}
\subsection{Results and their Implications}
Our experimental results for LWF CGs with 50 variables are partially (only for a few configurations of parameters) shown in Figures \ref{fig:503Mbs} and \ref{fig:503cglearn}. The other results are in Appendix B. We did not test
whether the faithfulness assumption holds for any of the networks, thus the results are indicative of the performance of the algorithms on arbitrary LWF CGs.

\textbf{Some highlights for Markov blanket discovery:}
(1) As shown in our experimental results, our proposed Markov blanket discovery algorithm, MBC-CSP, is as good as or even (slightly) better than others in many settings. 
(2) As expected, the recall of all algorithms increases with an increase in sample size.  Surprisingly, however, the other error measures worsen with an increase in sample size. A possible explanation could be that the correlation test is too aggressive and rejects variables that are in fact related in the ground truth model. 
(3) The significance level (p-value or $\alpha$ parameter) has a notable impact on the performance of algorithms. Except for precision, the lower the significance level, the better the performance.
(4) The fdrIAMB algorithm has the best precision, FPR, and ACC in small sample size settings, which is consistent with previously reported results~\citep{Pena08Mb}. This comes at the expense, however, of much worse recall.

\textbf{Some highlights for LWF CGs recovery:}
(1) As shown in our experimental results, our proposed Markov blanket based algorithm, MbLWF, is  as good as or even (slightly) better than LCD in many settings. The reason is that both LCD and MbLWF algorithms take advantage of local computations that make them equally robust against the choice of learning parameters.  (2) While our Markov blanket based algorithms have better precision and FPR, the LCD algorithm enjoys (slightly) better recall. The reason for this may be that the faithfulness
assumption makes the LCD algorithm search for a CG that represents all the independencies that are detected in the sample set. However, such a
CG may also represent many other independencies. Therefore, the LCD algorithm trades precision for recall. 
In other words, it seems that the faithfulness assumption makes the LCD algorithm overconfident
and aggressive, whereas under this
assumption MbLWF algorithms are more cautious, conservative, and more importantly more precise than the LCD algorithm. (3) Except for the fdrIAMB algorithm in small sample size, there is no meaningful difference among the performance of algorithms based on ACC. (4) The best SHD belongs to MBC-CSPLWF and LCD in small sample size settings, and to MBC-CSPLWF, fdrIAMB, and LCD in large sample size settings. (5) Constraint-based learning algorithms always have been criticized for their relatively high structural-error rate \citep{Triantafillou14}. However, as shown in our experimental results, the proposed Markov blanket based approach is, overall,  as good as or even better than the state-of the-art algorithm, i.e., LCD. One of the most important implications of this work is that there is much room for improvement to the constraint-based algorithms in general and Markov blanket based learning algorithms in particular, and hopefully this work will inspire other researchers to address this important class of algorithms. (6) Markov blankets of different variables can be learned independently from each other, and later merged and reconciled to produce a coherent LWF CG. This allows the parallel implementations for scaling up the task of learning chain graphs from data containing more than hundreds of variables, which is crucial for big data analysis tools. In fact, our proposed structure learning algorithms can be parallelized following \citep{scutari17}; see supplementary material for a detailed example. 

With the use of our generic algorithm (Algorithm \ref{alg:LWFMb}), the problem of structure learning is reduced to finding an efficient algorithm for Markov blanket discovery in LWF CGs. This greatly simplifies the structure-learning task and makes a wide range of inference/learning problems computationally tractable because they exploit locality. In fact, due to the causal interpretation of LWF CGs \citep{rs,Bhattacharya19}, discovery of Markov blankets in LWF CGs is significant because it can play an important role for estimating causal effects under unit dependence induced by a network represented by a CG model, when there
is uncertainty about the network structure. 

\section{DISCUSSION AND CONCLUSION}
An important novelty of local methods in general and Markov blanket recovery algorithms in particular for structure learning  is circumventing non-uniform graph connectivity. A chain graph may be non-uniformly dense/sparse. In a global learning framework, if a region is particularly dense, that region cannot be discovered quickly and many errors will result when  learning with a small sample. These errors propagate to remote regions in the chain graph including
those that are learnable accurately and fast with local methods. In contrast, local methods such as Markov blanket discovery algorithms are fast and accurate in the less dense regions.
In addition, when the dataset has tens or hundreds of thousands of variables, applying global discovery
algorithms that learn the full chain graph becomes impractical. In those cases,
Markov blanket based approaches that take advantage of local computations can be used for learning full LWF CGs. For this purpose, we extended the concept of Markov blankets to LWF CGs and we proposed a new algorithm, called MBC-CSP, for Markov blanket discovery in LWF CGs. We proved that GSMB and IAMB and its variants are still sound for Markov blanket discovery in LWF CGs under the faithfulness and causal sufficiency assumptions. This, in turn, enabled us to extend these algorithms to a new family of global structure learning algorithms based on Markov blanket discovery. As we have shown for the MBC-CSP algorithm, having an effective strategy for Markov blanket recovery in LWF CGs improves the quality of the learned Markov blankets, and consequently the learned LWF CG. 

As noticed by \cite{LiWang09}, the choice of which
performance parameter to optimize (equivalently, which error parameter to control) depends on the application, so
we reported on several performance parameters in our experiments. We plan to address the multiple hypotheses testing problem in the small sample case in future work.
An approach based on the theoretical work in \citep{Benjamini2001} that uses explicit control of error rates
was attempted and carried out in \citep{Wang2019}.

Another interesting direction for future work is answering the following question: Can we relax the faithfulness assumption and
develop a correct, scalable, and data efficient algorithm for learning Markov blankets in LWF CGs?

\subsubsection*{Acknowledgements}
This work has been supported by AFRL and DARPA (FA8750-16-2-0042).  This work is also partially supported by an ASPIRE grant from the Office of the Vice President for Research at the University of South Carolina. We appreciate the comments, suggestions, and questions from all anonymous reviewers and thank them for the careful reading of our paper.  
\bibliographystyle{aaai}
\bibliography{sample}

\newpage
\appendix
\section*{Appendix A: Correctness of Algorithm \ref{alg:LWFMb}}\label{appendixA}

We prove the correctness of the Algorithm \ref{alg:LWFMb} with following lemmas.
\begin{lemma}\label{lem1lwfpc}
After line 13 of Algorithm \ref{alg:LWFMb}, $G$ and $H$ have the
same adjacencies.
\end{lemma}
\begin{proof}
Consider any pair of nodes $A$ and $B$
in $G$. If $A\in ad_G(B)$, then $A\not\perp\!\!\!\perp B|S$ for all $S\subseteq V\setminus (A\cup B)$ by the faithfulness assumption. Consequently, $A\in ad_H(B)$ at all times. On the other hand, if $A\not\in ad_G(B)$ (equivalently $B\not\in ad_G(A)$), Algorithm \ref{alg2sepLWF} \citep{jv-pgm18} returns a set $Z\subseteq ad_H(A)\setminus B$ (or $Z\subseteq ad_H(B)\setminus A$) such that $A\perp\!\!\!\perp_p B|Z$.  This means there exist $0\le i\le |V_H|-2$ such that the edge $A-B$ is removed from $H$ in line 10. Consequently, $A\not\in ad_H(B)$ after line 13.
\begin{algorithm2e}[ht]
\caption{Minimal separation}\label{alg2sepLWF}
	\SetAlgoLined
	\KwIn{Two non-adjacent nodes $A, B$ in the LWF chain graph $G$.}
	\KwOut{Set $Z$, that is a minimal separator for $A, B$.}
	Construct $G_{An(A\cup B)}$\;
	Construct $(G_{An(A\cup B)})^m$\;
	Set $Z'$ to be $ne(A)$ (or $ne(B)$) in $(G_{An(A\cup B)})^m$\;
	\tcc{$Z'$ is a separator because, according to the local Markov property of an undirected graph, a vertex is conditionally independent of all other vertices in the graph, given its neighbors \citep{l}.}
	Starting from $A$, run BFS. Whenever a node in $Z'$ is met, mark it if it is not already marked, and do not continue along
	that path. When BFS stops, let $Z''$ be the set of nodes which are marked. Remove all markings\;
	Starting from $B$, run BFS. Whenever a node in $Z''$ is met, mark it if it is not already marked, and do not continue along
	that path. When BFS stops, let  $Z$ be the set of nodes which are marked\;
	\Return{$Z$}\;
\end{algorithm2e}
 
\end{proof}

\begin{lemma}\label{lem2lwfpc}
$G$ and $H^*$ have the same minimal complexes and adjacencies after line 22 of Algorithm \ref{alg:LWFMb}.
\end{lemma}
\begin{proof}
$G$ and $H^*$ have the same adjacencies by Lemma \ref{lem1lwfpc}. Now we show that any arrow that belongs to a minimal complex in $G$ is correctly oriented in line 18 of Algorithm \ref{alg:LWFMb}, in the sense that it is an arrow
with the same orientation in $G$. For this purpose, consider the following two cases:

\textbf{Case 1:} $u\to w\gets v$ is an induced subgraph in $G$. So, $u, v$ are not adjacent in $H$ (by Lemma \ref{lem1lwfpc}), $u-w\in H^*$ (by Lemma \ref{lem1lwfpc}), and $u\not\perp\!\!\!\perp_p v|(S_{uv}\cup \{w\})$ by the faithfulness assumption. So, $u - w$ is oriented as $u\to w$ in $H^*$ in line 15. Obviously, we will not orient it as $w\to u$.

\textbf{Case 2:} $u\to w-\cdots-z\gets v$, where $w\ne z$ is a minimal complex in $G$. So, $u, v$ are not adjacent in $H$ (by Lemma \ref{lem1lwfpc}), $u-w\in H^*$ (by Lemma \ref{lem1lwfpc}), and $u\not\perp\!\!\!\perp_p v|(S_{uv}\cup \{w\})$ by the faithfulness assumption. So, $u - w$ is oriented as $u\to w$ in $H^*$ in line 15. Since $u\in S_{vw}$ and $w\perp\!\!\!\perp_p v|(S_{wv}\cup \{u\})$ by the faithfulness assumption  so $u,v$, and $w$ do not satisfy the conditions and hence we will not orient $u - w$ as $w\to u$.

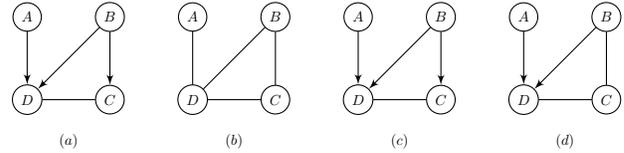
\begin{figure}[ht]
    \centering
	\[\begin{tikzpicture}[scale=.55, transform shape]
	\tikzset{vertex/.style = {shape=circle,draw,minimum size=1.5em}}
	\tikzset{edge/.style = {->,> = latex'}}
	\node[vertex] (o) at  (5,-2) {$B$};
	\node[vertex] (p) at  (3,-4) {$D$};
	\node[vertex] (q) at  (3,-2) {$A$};
	\node[vertex] (r) at  (5,-4) {$C$};
	\node (s) at (4, -5) {$(a)$};
	\draw[edge] (q) to (p);
	\draw[edge] (o) to (r);
	\draw[edge] (o) to (p);
	\draw (p) to (r);
	
	\node[vertex] (u) at  (9,-2) {$B$};
	\node[vertex] (v) at  (7,-4) {$D$};
	\node[vertex] (x) at  (7,-2) {$A$};
	\node[vertex] (w) at  (9,-4) {$C$};
	\node (y) at (8, -5) {$(b)$};
	\draw (x) to (v);
	\draw (u) to (v);
	\draw (u) to (w);
	\draw (v) to (w);
	
	\node[vertex] (a) at  (13,-2) {$B$};
	\node[vertex] (b) at  (11,-4) {$D$};
	\node[vertex] (c) at  (11,-2) {$A$};
	\node[vertex] (d) at  (13,-4) {$C$};
	\node (e) at (12, -5) {$(c)$};
	\draw[edge] (a) to (d);
	\draw[edge] (a) to (b);
	\draw[edge] (c) to (b);
	\draw (b) to (d);
	
	\node[vertex] (f) at  (17,-2) {$B$};
	\node[vertex] (g) at  (15,-4) {$D$};
	\node[vertex] (h) at  (15,-2) {$A$};
	\node[vertex] (i) at  (17,-4) {$C$};
	\node (j) at (16, -5) {$(d)$};
	\draw (f) to (i);
	\draw[edge] (h) to (g);
	\draw[edge] (f) to (g);
	\draw (g) to (i);
	
	\end{tikzpicture}\]
    \caption{(a) The LWF CG $G$, (b) the skeleton of $G$, (c) $H^*$ before executing the line 22 in Algorithm \ref{alg:LWFMb}, and (d) $H^*$ after executing the line 22 in Algorithm \ref{alg:LWFMb}. 
}
    \label{patternlwfpc}
\end{figure}

Consider the chain graph $G$ in Figure \ref{patternlwfpc}(a). After applying the skeleton recovery of Algorithm \ref{alg:LWFMb}, we obtain $H$, the skeleton of $G$, in Figure \ref{patternlwfpc}(b). In the execution of the complex recovery of
Algorithm \ref{alg:LWFMb}, when we pick $A,B$ in line 15 and $C$ in line 16, we have $A\perp\!\!\!\perp B|\varnothing$, that is, $S_{AB}=\varnothing$, and
find that $A\not\perp\!\!\!\perp B|C$. Hence we orient $B-C$ as $B\to C$ in line 18, which is not a complex arrow in $G$.
Note that we do not orient $C-B$ as $C\to B$: the only chance we might do so is when $u = C, v = A$, and $w = B$ in the inner loop of the complex recovery of Algorithm \ref{alg:LWFMb}, but we have $B\in S_{AC}$ and the condition in line 17 is not
satisfied. Hence, the graph we obtain before the last step of complex recovery in Algorithm \ref{alg:LWFMb} must be the one given in
Figure \ref{patternlwfpc}(c), which differs from the recovered pattern in Figure \ref{patternlwfpc}(d). This illustrates the necessity of the last step of complex recovery in Algorithm \ref{alg:LWFMb}. To see how the edge $B\to C$ is removed in the last step of complex recovery in Algorithm \ref{alg:LWFMb}, we observe that, if we follow the procedure described in the comment after line 22 of Algorithm \ref{alg:LWFMb}, the only chance
that $B\to C$ becomes one of the candidate complex arrow pair is when it is considered together with $A\to D$. However, the only undirected path between $C$ and $D$ is simply $C-D$ with $D$ adjacent to $B$.
Hence $B\to C$ stays unlabeled and will finally get removed in the last step of complex recovery in Algorithm \ref{alg:LWFMb}.

Consequently, $G$ and $H^*$ have the same minimal complexes and adjacencies after line 22.
\end{proof}

\section*{Appendix B: More Experimental Results}\label{appendixB}
\textbf{Data Generation Procedure}
First we explain the way in which the random LWF CGs and random samples are generated.
Given a vertex set $V$, let $p = |V|$ and $N$ denote the average degree of edges (including undirected, pointing out, and pointing in) for each vertex. We generate a random LWF CG on $V$ as
follows:

(1) Order the $p$ vertices and initialize a $p\times p$ adjacency matrix $A$ with zeros;

(2) For each element in the lower triangle part of $A$, set it to be a random number generated from
a Bernoulli distribution with probability of occurrence $s = N/(p-1)$;

(3) Symmetrize $A$ according to its lower triangle;

(4) Select an integer $k$ randomly from $\{1,\dots,p\}$ as the number of chain components;

(5) Split the interval $[1, p]$ into $k$ equal-length subintervals $I_1,\dots,I_k$ so that the set of variables
falling into each subinterval $I_m$ forms a chain component $C_m$; 

(6) Set $A_{ij} = 0$ for any $(i, j)$ pair such that $i \in I_l, j \in I_m$ with $l > m$.

This procedure yields an adjacency matrix $A$ for a chain graph with $(A_{ij} = A_{ji} = 1)$ representing an undirected edge between $V_i$ and $V_j$ and $(A_{ij} =1,  A_{ji} =0)$ representing a directed edge
from $V_i$ to $V_j$. Moreover, it is not difficult to see that $\mathbb{E}[\textrm{vertex degree}] = N$, where an adjacent vertex can
be linked by either an undirected or a directed edge.

Given a randomly generated chain graph $G$ with ordered chain components $C_1,\dots,C_k$, we generate a Gaussian distribution on it via the $\mathsf{rnorm.cg}$ function from the \href{http://www2.uaem.mx/r-mirror/web/packages/lcd/lcd.pdf}{LCD} R package.

\textbf{Metrics for Evaluation}
We evaluate the performance of the proposed algorithms in terms of the six measurements that are commonly used~\citep{Colombo2014,mxg,Tsamardinos2006} for constraint-based learning algorithms: (a) the true positive
rate (TPR) (also known as sensitivity, recall, and hit rate), (b) the false positive rate (FPR) (also known as fall-out), (c) the true discovery rate (TDR) (also known as precision or positive predictive value), (d) accuracy (ACC) for the skeleton, (e) the structural Hamming distance (SHD) (this is the metric described in \cite{Tsamardinos2006} to  compare the
structure of the learned and the original graphs), and (f) run-time for the pattern recovery algorithms. In short, $TPR=\frac{\textrm{true positive } (TP)}{\textrm{the number of real positive cases in the data } (Pos)}$ is the ratio of  the number of correctly identified edges over total number of edges, $FPR=\frac{\textrm{false positive }(FP)}{\textrm{the number of real negative cases in the data }(Neg)}$ is the ratio of the number of incorrectly identified edges over total number of gaps, $TDR=\frac{\textrm{true positive } (TP)}{\textrm{the total number of edges in the recovered CG}}$ is the ratio of  the number of correctly identified edges over total number of edges (both in estimated graph), $ACC=\frac{\textrm{true positive }(TP) +\textrm{ true negative }(TN)}{Pos+Neg}$, and
$SHD$ is the number of legitimate operations needed to change the current resulting graph to the true CG,
where legitimate operations are: (a) add or delete an edge and (b) insert, delete or reverse an edge
orientation. In principle, a large TPR, TDR, and ACC, a small FPR and SHD indicate good performance.
\begin{figure}[h]
	\centering
	\includegraphics[width=.75\linewidth]{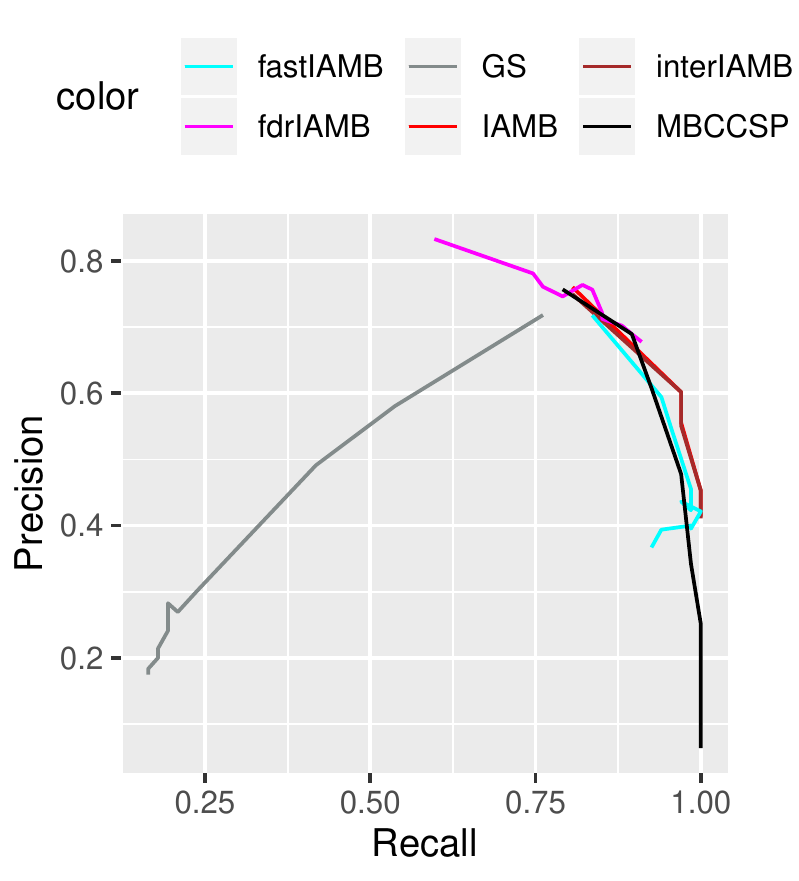}
	\includegraphics[width=.75\linewidth]{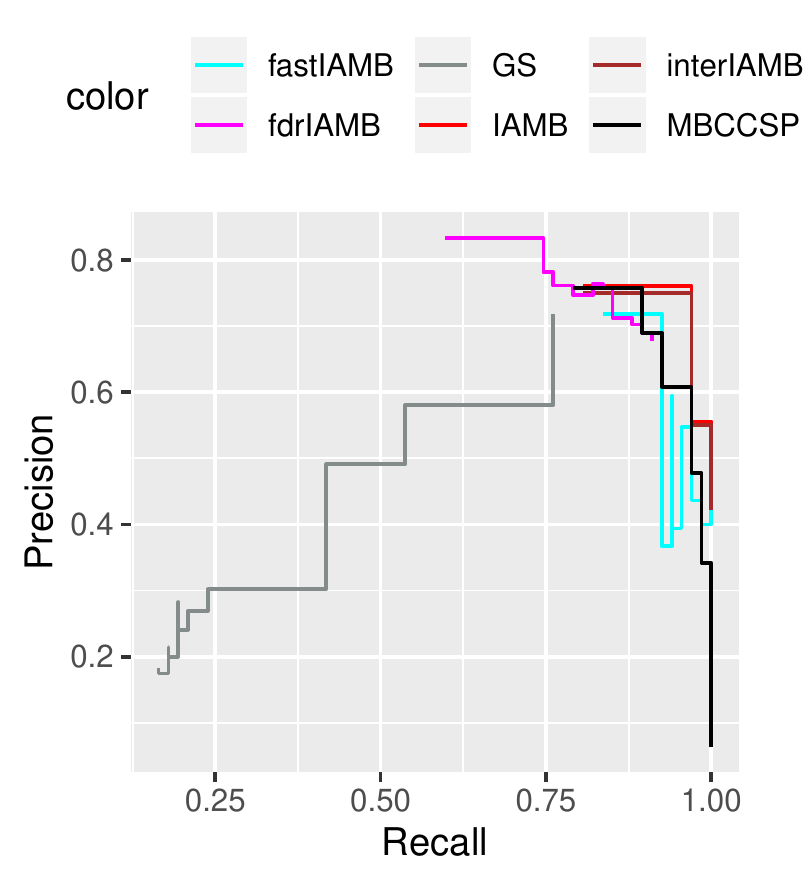}
	\caption{Precision-Recall ROC curves (pathwise and stepwise) along several different alpha values ($\alpha\in(0.005,0.05,0.1,0.2,0.3,0.4,0.5,0.6,0.7,0.8,0.9)$) for a randomly generated Gaussian CG model with 50 variables  and N = 3; these curves show the precision-recall trade-off.}
	\label{fig:PRcurve}
\end{figure}

\begin{figure}
	\centering
	\includegraphics[scale=.475,page=1]{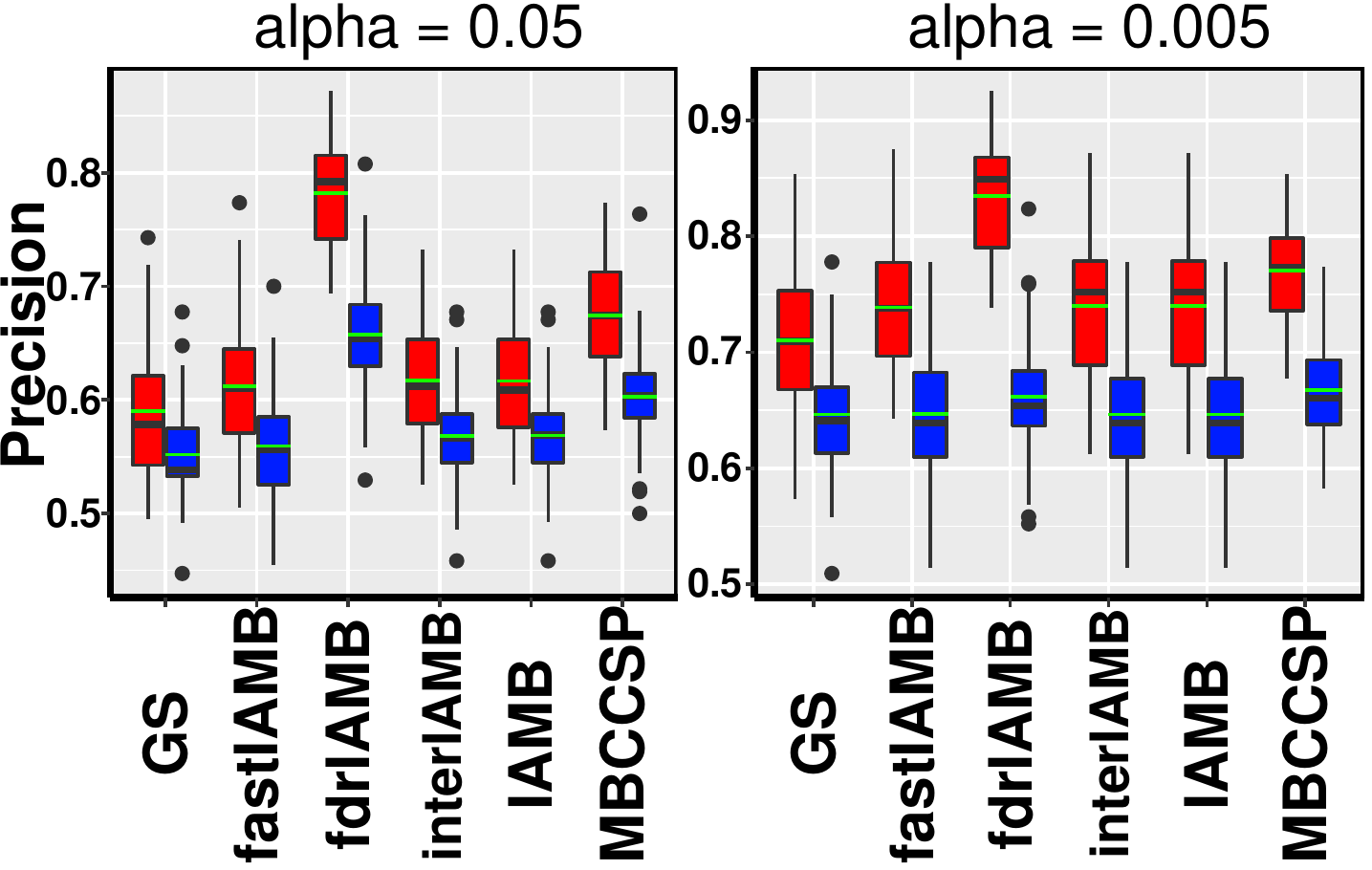}
	\includegraphics[scale=.475,page=2]{images/Mbsboxplots502.pdf}
	\includegraphics[scale=.475,page=3]{images/Mbsboxplots502.pdf}
	\includegraphics[scale=.475,page=4]{images/Mbsboxplots502.pdf}
	\includegraphics[scale=.475,page=5]{images/Mbsboxplots502.pdf}
	\caption{Performance of Markov blanket recovery algorithms for randomly generated Gaussian chain graph models:
		 over 30 repetitions with 50 variables  correspond to N = 2. The green line in a box indicates the mean of that group.}
	\label{fig:502Mbs}
\end{figure}

\begin{figure}
	\centering
	\includegraphics[scale=.475,page=1]{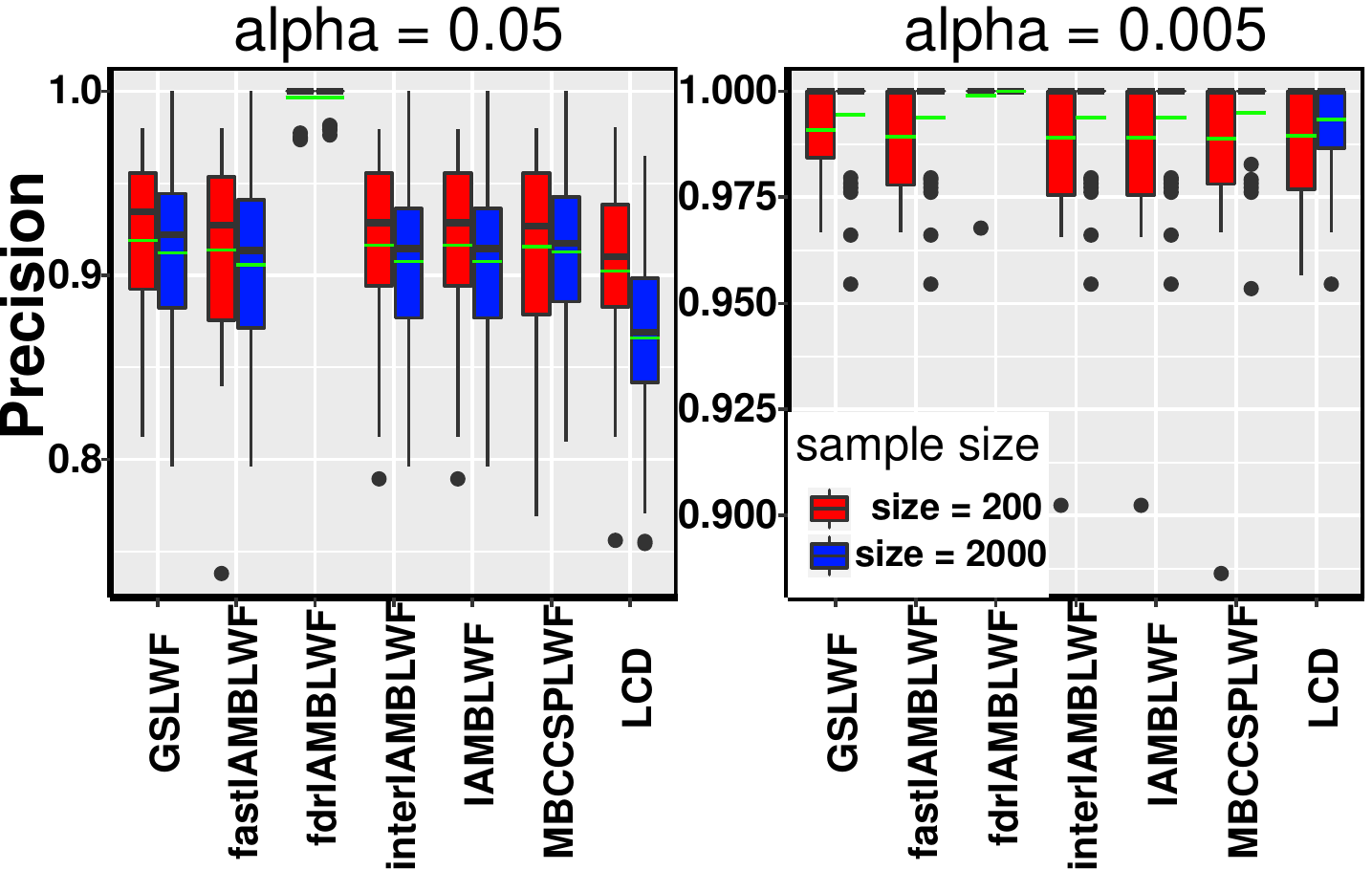}
	\includegraphics[scale=.475,page=2]{images/MBLWFCG502.pdf}
	\includegraphics[scale=.475,page=3]{images/MBLWFCG502.pdf}
	\includegraphics[scale=.475,page=4]{images/MBLWFCG502.pdf}
	\includegraphics[scale=.475,page=5]{images/MBLWFCG502.pdf}
	\caption{Performance of LCD and MbLWF algorithms for randomly generated Gaussian chain graph models:
		 over 30 repetitions with 50 variables  correspond to N = 2. The green line in a box indicates the mean of that group.}
	\label{fig:502cglearn}
\end{figure}

\begin{figure}
	\centering
	\includegraphics[scale=.475,page=1]{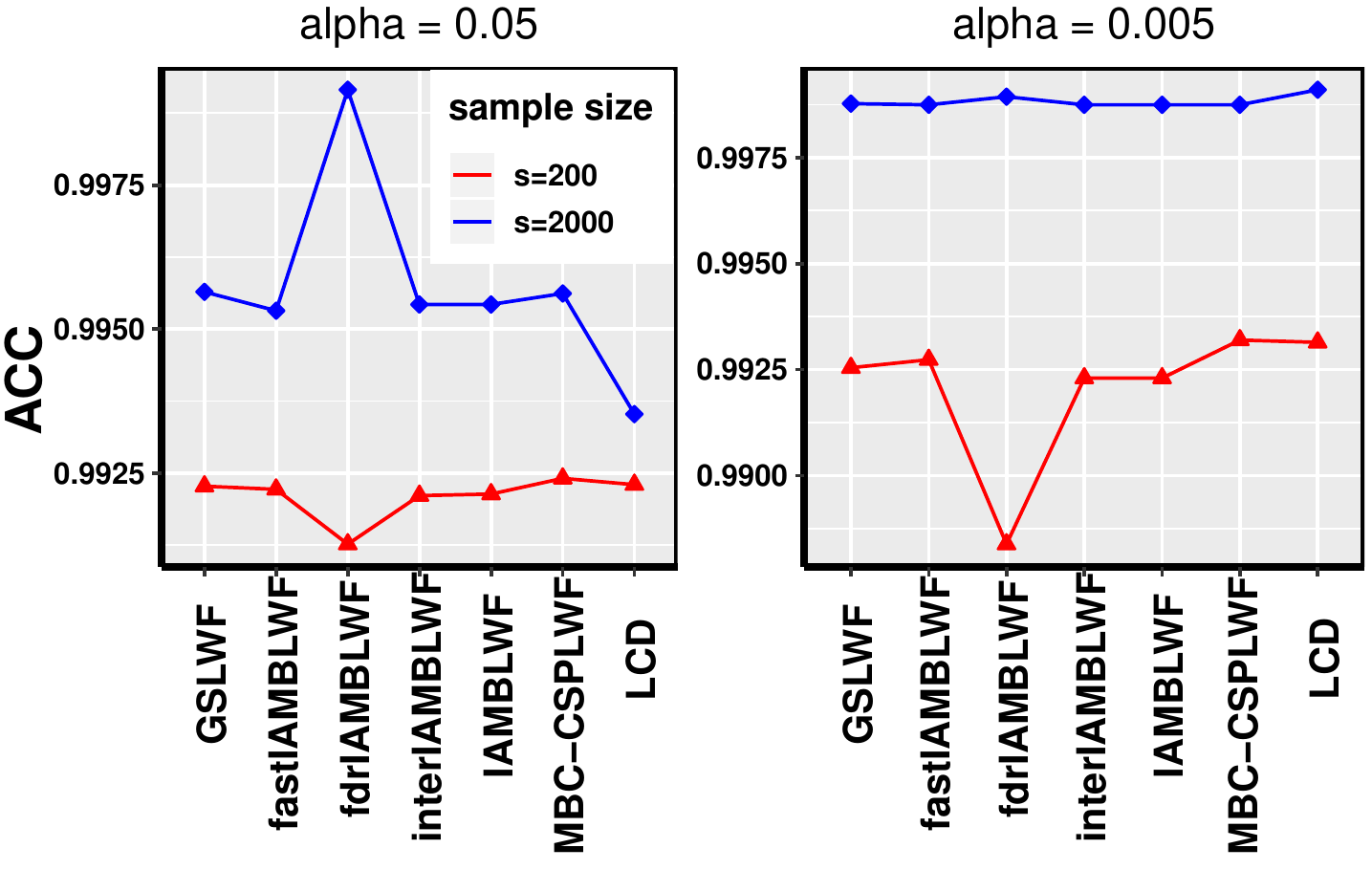}
	\includegraphics[scale=.475,page=2]{images/plots502uai2020.pdf}
	\includegraphics[scale=.475,page=3]{images/plots502uai2020.pdf}
	\includegraphics[scale=.475,page=4]{images/plots502uai2020.pdf}
	\includegraphics[scale=.475,page=5]{images/plots502uai2020.pdf}
	\caption{Performance of LCD and MbLWF algorithms for randomly generated Gaussian chain graph models:
		average over 30 repetitions with 50 variables  correspond to N = 2.}
	\label{fig:502}
\end{figure}

\begin{figure}
	\centering
	\includegraphics[scale=.475,page=1]{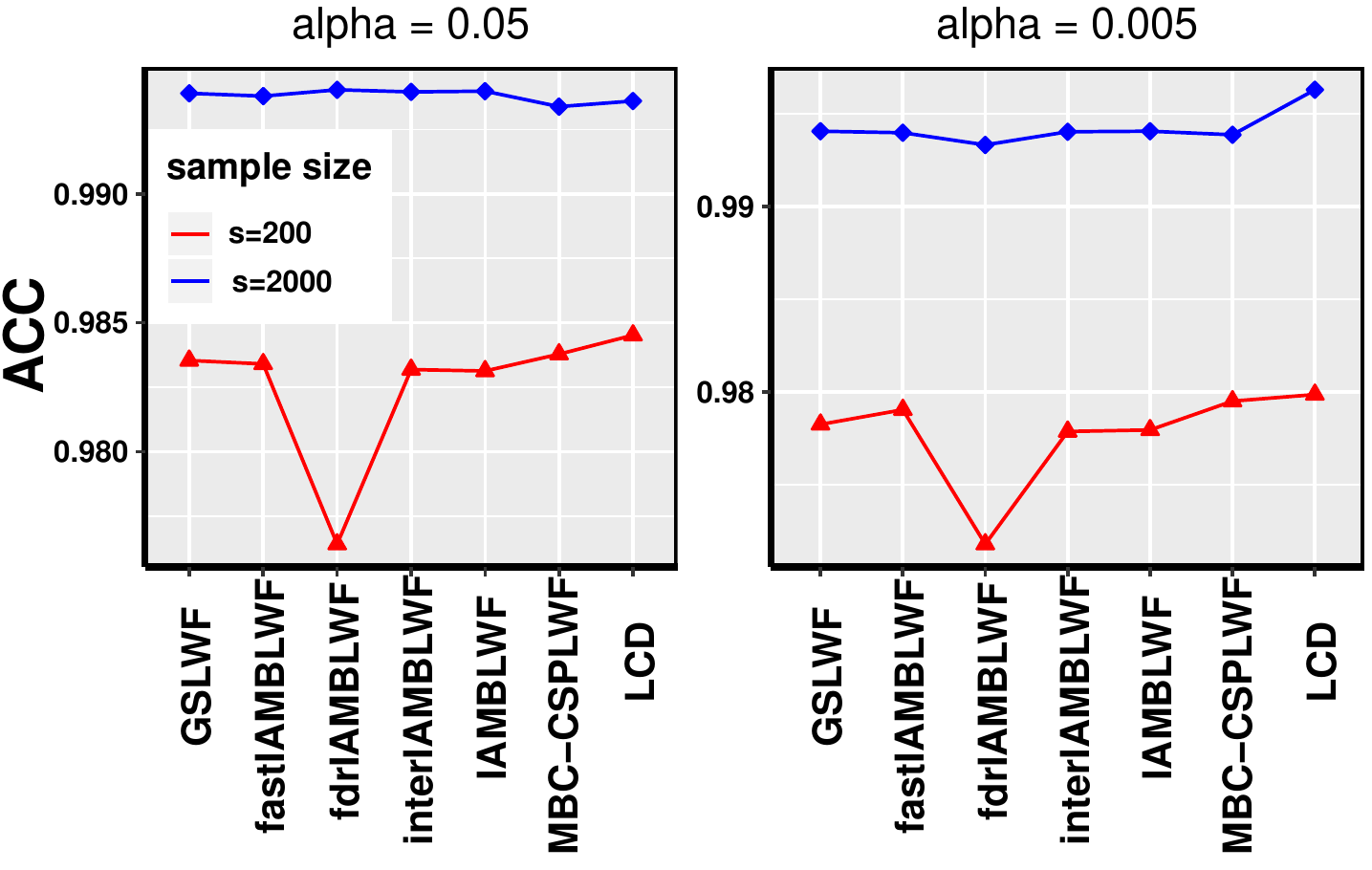}
	\includegraphics[scale=.475,page=2]{images/plots503uai2020.pdf}
	\includegraphics[scale=.475,page=3]{images/plots503uai2020.pdf}
	\includegraphics[scale=.475,page=4]{images/plots503uai2020.pdf}
	\includegraphics[scale=.475,page=5]{images/plots503uai2020.pdf}
	\caption{Performance of LCD and MbLWF algorithms for randomly generated Gaussian chain graph models:
		average over 30 repetitions with 50 variables  correspond to N = 3.}
	\label{fig:503}
\end{figure}

\end{document}